\documentclass[nohyperref]{article}

\usepackage{microtype}
\usepackage{graphicx}
\usepackage{subfigure}
\usepackage{booktabs} 
\usepackage{graphics}
\usepackage{hyperref}

\usepackage[accepted]{icml2022}

\usepackage{amsmath}
\usepackage{amssymb}
\usepackage{mathtools}
\usepackage{amsthm}
\usepackage{microtype}
\usepackage{graphicx}
\usepackage{subfigure}
\usepackage{makeidx}
\usepackage{multicol}
\usepackage{balance}
\usepackage[utf8]{inputenc} 
\usepackage[T1]{fontenc}
\usepackage{wrapfig}
\usepackage{multirow}
\usepackage{enumitem}
\usepackage{array}
\usepackage{comment}
\usepackage{caption}
\usepackage{adjustbox}
\usepackage{mwe}

\usepackage{color}

\usepackage[capitalize,noabbrev]{cleveref}

\theoremstyle{plain}
\newtheorem{theorem}{Theorem}[section]
\newtheorem{proposition}[theorem]{Proposition}

\theoremstyle{definition}
\newtheorem{definition}[theorem]{Definition}

\theoremstyle{remark}

\usepackage[textsize=tiny]{todonotes}

\icmltitlerunning{Approximately Equivariant Networks for Imperfectly Symmetric Dynamics}
\begin{document}

\twocolumn[
\icmltitle{Approximately Equivariant Networks for Imperfectly Symmetric Dynamics}



\icmlsetsymbol{equal}{*}

\begin{icmlauthorlist}
\icmlauthor{Rui Wang}{equal,sd}
\icmlauthor{Robin Walters}{equal,neu}
\icmlauthor{Rose Yu}{sd}
\end{icmlauthorlist}

\icmlaffiliation{sd}{University of California San Diego}
\icmlaffiliation{neu}{Northeastern University}

\icmlcorrespondingauthor{Rui Wang}{ruw020@ucsd.edu}

\icmlkeywords{Machine Learning, ICML}

\vskip 0.3in
]



\printAffiliationsAndNotice{\icmlEqualContribution} 

\begin{abstract}
Incorporating symmetry as an inductive bias into neural network architecture has led to improvements in generalization, data efficiency, and physical consistency in dynamics modeling. Methods such as CNNs or equivariant neural networks use weight tying to enforce symmetries such as shift invariance or rotational equivariance.  However, despite the fact that physical laws obey many symmetries, real-world dynamical data rarely conforms to strict mathematical symmetry either due to noisy or incomplete data or to symmetry breaking features in the underlying dynamical system.  We explore  approximately equivariant networks which are biased towards preserving symmetry but are not strictly constrained to do so.  By relaxing equivariance constraints, we find that our models can outperform both baselines with no symmetry bias and baselines with overly strict symmetry in both simulated turbulence domains and real-world multi-stream jet flow.   
\end{abstract}

\section{Introduction}

Symmetry and equivariance are fundamental to the success of deep learning \cite{bronstein2021geometric}. The canonical examples are translation invariance in convolutional layers \cite{fukushima1982neocognitron,lecun1989backpropagation,krizhevsky2012imagenet}, and permutation invariance in graph neural networks \cite{bruna2013spectral,battaglia2018relational,maron2018invariant}. Recently, equivariant networks, which encode symmetry information in network architectures,  have gained significant attention for modeling structured and complex data \cite{ravanbakhsh2017equivariance,zaheer2017deep,kondor2018generalization,cohen2016group,worrall2017harmonic,thomas2018tensor,cohen2018general,maron2020learning, Walters2021ECCO}.

However, existing equivariant networks assume perfect symmetry in the data. The network is approximating a function that is strictly invariant or equivariant under a given group action. However, real-world data are very rarely perfectly symmetric. For example, in turbulence modeling, even though the governing equations of turbulence satisfy many different symmetries such as scale invariance \cite{holmes2012turbulence}, effects such as varying external forces, certain boundary conditions, or the presence of missing values would break these symmetries to varying degrees. This significantly hinders the potential applications of equivariant networks. Approximately equivariant networks could outperform both strictly equivariant networks and highly flexible models in learning many dynamics in the real world, as shown in Figure \ref{vis:salesman}.

\begin{figure}[t]
	\centering
	\hspace*{-0.3cm}
	\includegraphics[width=0.5\textwidth]{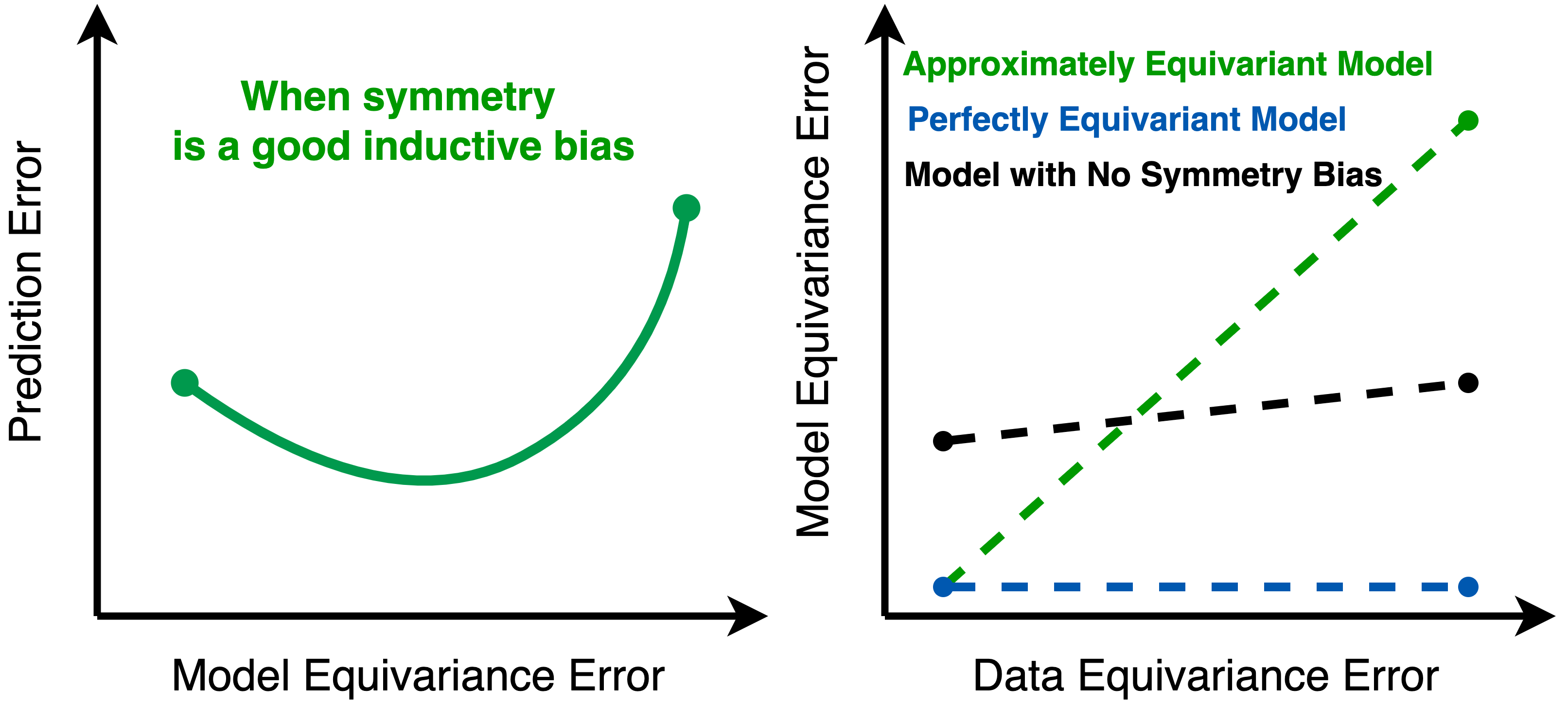}
	\caption{\textbf{Left}: When symmetry is a good inductive bias, prediction performance increases as equivariance or invariance is imposed in the model. But real-world data is very rarely perfectly symmetric, and so relaxing the strict constraint in equivariant networks to balance inductive bias and expressivity can further improve predictive performance. \textbf{Right}: Highly flexible models have trouble achieving zero equivariance error without the guide of appropriate symmetry biases when the data is symmetric. Perfectly equivariant models maintain zero equivariance error, which is overly restricted when data is not perfectly symmetric. An ideal model for real world dynamics should be approximately equivariant and automatically learn the correct amount of symmetry in the data.}
	\label{vis:salesman}
\end{figure}

Relaxing the rigid assumption in equivariant networks to balance inductive bias and expressivity in deep learning has been the pursuit of a few recent works. For example, \citet{Elsayed2020Revisiting} showed that spatial invariance can be overly restrictive, and relaxing spatial weight sharing in standard convolution can improve image classification accuracy.    \citet{d2021convit} enforce a convolutional inductive bias in self-attention layers at \textit{initialization} to improve vision Transformers. Residual Pathway Priors \cite{finzi2021residual} convert hard architectural constraints into soft priors by placing a higher likelihood on the ``residual''.  The residual explains the difference between the structure in the data and the inductive bias encoded by an equivariant model. \citet{wang2021equivariant} proposed {Lift Expansion} which factorizes the data into equivariant and nonequivariant components and models them jointly. Despite progress, a formal definition of approximate symmetry does not exist. While existing research focuses on translation symmetry, the rich groups of symmetry in high-dimensional dynamics learning problems are unexplored.

In this paper, we first define approximate symmetry. It gives rise to a new class of approximately equivariant networks that avoid stringent symmetry constraints while maintaining favorable inductive biases for learning. Specifically, we generalize the weight relaxation scheme originally proposed by \cite{Elsayed2020Revisiting}. We study three symmetries that are common in dynamics: rotation $\mathrm{SO(2)}$, scaling $\mathbb{R}_{>0}$, and Euclidean $E(2)$. For group convolution, we relax the weight-sharing scheme by expressing the kernel as a weighted combination of multiple filter banks. For steerable CNNs, we introduce dependencies on the input to the kernel basis. We apply our approximate symmetry networks to the challenging problem of forecasting fluid flow and observe significant improvements for both synthetic and real-world datasets \footnote{We open-source our code \url{https://github.com/Rose-STL-Lab/Approximately-Equivariant-Nets}}. Our contributions include:
\begin{itemize}[noitemsep,topsep=2pt,parsep=5pt,partopsep=2pt,leftmargin=*]
\item We formally characterize the notion of approximate equivariance, which interpolates between no inductive bias and a strong inductive bias from equivariance. 
\item We introduce a new class of approximately equivariant networks for modeling imperfectly symmetric dynamics by relaxing equivariance constraints.
\item We demonstrate that our approximately equivariant models can outperform baselines with
no symmetry bias, baselines with overly strict symmetry, and SoTA approximately equivariant models in both simulated smoke simulations and real experimental jet flow data.
\end{itemize}






\section{Mathematical Preliminaries}

\subsection{Equivariant Functions and Neural Networks}

Equivariant neural networks incorporate an explicit symmetry constraint.  They are typically employed when a priori knowledge, such as first principles from physics, imply the ground truth function also respects a symmetry. 

\paragraph{Equivariance and Invariance.}
Formally, a function $f \colon X \to Y$ may be described as respecting the symmetry coming from a group $G$ using the notion of equivariance.  Assume that an input group representation $\rho_{\text{in}}$ of $G$ acts on $X$ and an output representation $\rho_{\text{out}}$ acts on $Y$. We say a function $f$ is \textbf{$G$-equivariant} if 
\begin{equation}
    f( \rho_{\text{in}}(g)(x)) = \rho_{\text{out}}(g) f(x) \label{eq:strictequivariance}
\end{equation}
for all $x \in X$ and $g \in G$. The function $f$ is \textbf{$G$-invariant} if $f( \rho_{\text{in}}(g)(x)) = f(x)$ for all $x \in X$ and $g \in G$.  This is a special case of equivariance for the case $\rho_{\mathrm{out}}(g) = 1$.

\paragraph{Strictly Equivariant Neural Networks.}
Given an equivariant $f \colon X \to Y$, learning can be accelerated by optimizing within a model class of functions $\lbrace f_\theta \rbrace$ which are restricted to be equivariant.  Since the composition of equivariant functions is again equivariant, in general a neural network will be strictly equivariant if all of its layers, linear, nonlinear, pooling, aggregation, and normalization, are equivariant.  Most of the variation and challenge in this area is in designing trainable equivariant linear layers.  Two strategies, involving weight sharing and weight tying, are $G$-convolution and $G$-steerable CNN. 
See \citet{bronstein2021geometric} for more details.

\paragraph{$\mathbf{G}$-Equivariant Group Convolution.} 
A $G$-equivariant group convolution \cite{cohen2016group} takes as input a $c_{\mathrm{in}}$-dimensional feature map $f \colon G \to \mathbb{R}^{c_{\mathrm{in}}}$ and convolves it with a kernel $\Psi \colon G \to \mathbb{R}^{c_\mathrm{out} \times c_\mathrm{in}}$ over a group $G$,
\begin{equation}
[f\star_G\Psi](g) = \sum_{h\in G}f(h)\Psi(g^{-1}h).
\label{eqn:group_conv}
\end{equation}
Here, we assume $G$ finite, however, $G$ may also be taken to be compact if the sum is replaced with an integral.  Group convolution achieves equivariance by weight sharing since the kernel weight $\Psi$ at $(g,h)$ depends only on $g^{-1}h$ and thus pairs with equal $g^{-1}h$ share weights. 

\paragraph{$\mathbf{G}$-Steerable Convolution.} \cite{Cohen2020STEERABLE}
Let $f$ be the input feature map $f \colon \mathbb{R}^2 \to \mathbb{R}^{c_\mathrm{in}}$.  Fix a subgroup $H \subset O(2)$, which acts on $\mathbb{R}^2$ by matrix multiplication and on the
 input and output channel spaces $\mathbb{R}^{c_\mathrm{in}}$ and $\mathbb{R}^{c_\mathrm{out}}$ by representations $\rho_{\mathrm{in}} \colon G \to \mathbb{R}^{c_\mathrm{in} \times c_\mathrm{in}}$ and $\rho_{\mathrm{out}} \colon G \to \mathbb{R}^{c_\mathrm{out} \times c_\mathrm{out}}$ respectively.   

We may convolve $f$ with a matrix-valued kernel $\Phi \colon \mathbb{R}^2 \to \mathbb{R}^{c_\mathrm{out} \times c_\mathrm{in}}$.
In practice, we discretize the input as $f \colon \mathbb{Z}^2 \to \mathbb{R}^{c_\mathrm{in}}$ and kernel $\Phi \colon \mathbb{Z}^2 \to \mathbb{R}^{c_\mathrm{out} \times c_\mathrm{in}}$ and compute the $H$-action by interpolation after rotation.
By \cite{weiler20183d}, the standard 2D convolution $f \star_{\mathbb{Z}^2} \Phi$ is $H$-equivariant and $\mathbb{Z}^2$-translation equivariant when
 \begin{equation}
 \Phi(hx) = \rho_{\text{out}}(h) \Phi(x) \rho_{\text{in}}(h^{-1}), \ \forall h \in H. \label{steerable}
 \end{equation}
 Again, $\Phi(hx)$ is computed based on the group and the choice of input and output representations.
 This linear constraint induces dependence in the weights, which is called weight tying. Solving for a basis of solutions to \eqref{steerable} gives an equivariant kernel basis $\{\Phi_l\}_{l=1}^L$ which can be combined using trainable coefficients $\Phi = \sum_{l=1}^L w_l \Phi_l$ to learn any element of the solution space formed by (3).

\vspace{-5pt}
\subsection{Approximate Equivariance}
\begin{figure}[htb!]
	\centering
	\includegraphics[width=0.48\textwidth]{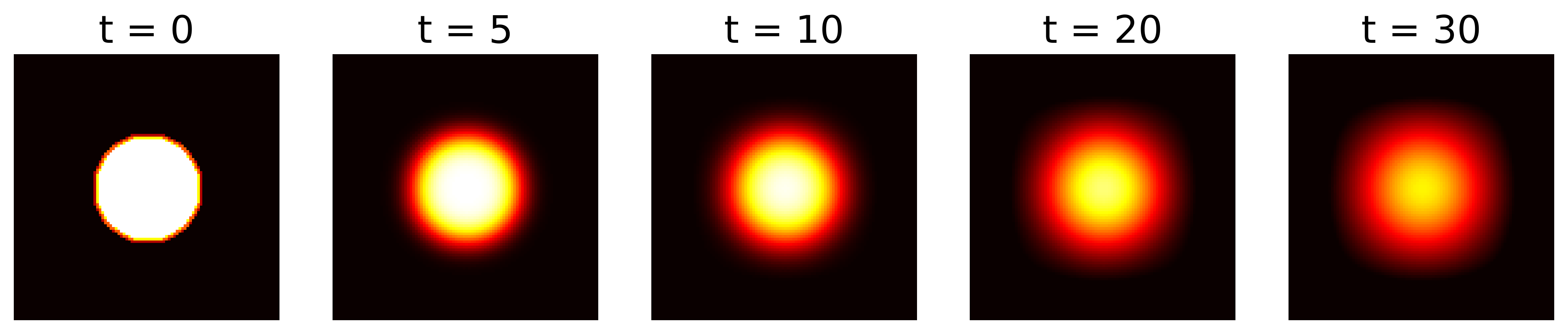}
	\includegraphics[width=0.48\textwidth]{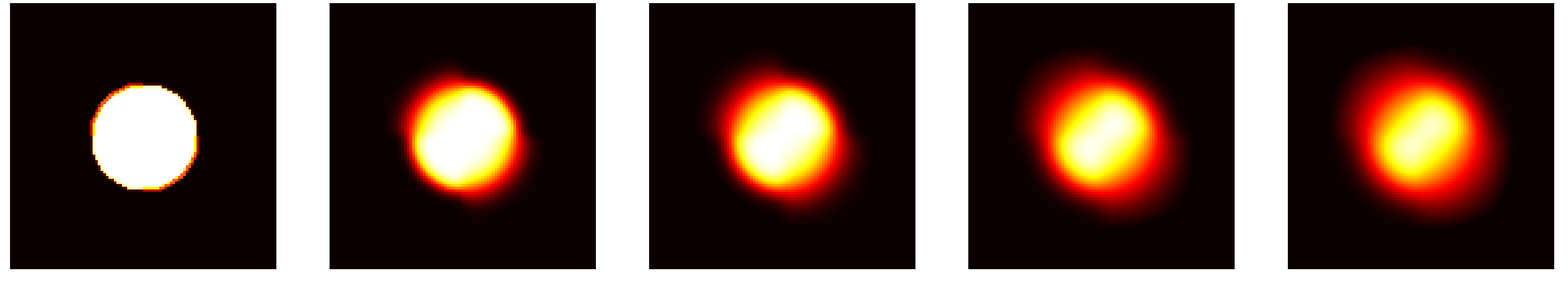}
	\caption{Simulated diffusion of heat in a metal plate with (top) uniform diffusion coefficient resulting in perfect symmetry and (bottom) slightly varying diffusion coefficient resulting in approximate symmetry.}
	\label{vis:heat}
\end{figure}
    


Real-world dynamics data may not satisfy the strict equivariance as in \eqref{eq:strictequivariance}.  However, since many of the governing equations contain symmetry, the resulting system may still be approximately equivariant, as defined below.  For example, while the heat equation itself is fully rotationally symmetric, in practice, imperfections in the thickness of the metal or the composition of the metal can lead to imperfect symmetry, as shown in Figure \ref{vis:heat}.
Below, we give the formal definition of approximate symmetry:
\begin{definition}[Approximate Equivariance]\label{def_app}
Let $f \colon X \to Y$ be a function and $G$ be a group. Assume that $G$ acts on $X$ and $Y$ via representations $\rho_{X}$ and $\rho_{Y}$. We say $f$ is $\epsilon$-approximately $G$-equivariant if for any $g \in G$, 
\[\vspace{-6pt}
    \|f(\rho_{X}(g)(x)) - \rho_{Y}(g)f(x)\| \leq \epsilon.
\]
\end{definition}
Note that strictly equivariant functions are $\epsilon = 0$ approximately equivariant.


\section{Approximately Equivariant Networks }
Symmetry in equivariant networks is enforced by strict constraints on the weights. Here we propose relaxing weight-sharing and weight-tying to model approximate symmetries. 
\citet{Elsayed2020Revisiting} showed that relaxing spatial weight sharing in standard convolution neural nets can improve image classification accuracy.  Whereas 2D convolutions are shift equivariant, this relaxed 2D convolution is only approximately equivariant. Our method generalizes this approach to other symmetry groups, including rotation $\mathrm{SO(2)}$, scaling $\mathbb{R}_{>0}$, and Euclidean $E(2)$.  Specifically, we relax the strict weight-sharing and weight-tying constraints in both group convolution and steerable CNNs. 


\subsection{Relaxed Group Convolution}
The $G$-equivariance of group convolution results from the shared kernel $\Psi(g^{-1}h)$ in \eqref{eqn:group_conv}.  To relax this and consequently relax the $G$-equivariance, we replace the single kernel $\Psi$ with a set of kernels $\lbrace \Psi_l \rbrace_{l=1}^L$. We define the new kernel $\Psi$ as a linear combination of $\Psi_l$ with coefficients that vary with $h$.  Thus,  we introduce symmetry-breaking dependence on the specific pair $(g,h)$, 
\begin{equation}
    \Psi(g,h) = \sum_{l=1}^L w_l(h) \Psi_l(g^{-1}h). \label{equ:relaxed_group_conv}
\end{equation}
We define the relaxed group convolution by multiplication with $\Psi$ as such 
\begin{equation}
\begin{aligned}
  [f \tilde{\star}_G \Psi](g) &= \sum_{h\in G}f(h)\Psi(g, h)\\
  &= \sum_{h\in G}\sum_{l=1}^L f(h) w_l(h) \Psi_l(g^{-1}h).
\end{aligned}
\end{equation}
By varying the number of kernels $L$, we can control the degree of equivariance. Small $L$ imposes stronger symmetry, while large $L$ gives more flexibility. In our experiments, we found that $L = 3$ gave the best prediction performance in most cases. The weights $w_i(h) \in \mathbb{R}$ and the kernels $\Psi_l(g^{-1}h) \in \mathbb{R}^{c_\mathrm{out} \times c_\mathrm{in}}$ can be learnt from data.  Relaxed group convolution reduces to group convolution and is fully equivariant if and only if $g_1^{-1}h_1 = g_2^{-1} h_2$ implies $\Psi(g_1,h_1) = \Psi(g_2,h_2)$.  In particular, this occurs if $w_l(h_1) =  w_l(h_2)$ for all $h_1,h_2 \in G$ and for all $l=1,\cdots,L$.

In dynamics learning, we consider velocity vectors as inputs.  To apply group convolution over the discrete rotation group $C_n$, we first lift these velocity vectors to feature maps $f \colon C_n \to \mathbb{R}$ as described in \citet{Walters2021ECCO} Table 1. 
Given $v = (a,b) \in \mathbb{R}^2$, for $i \in C_n$ we define $f(i) = c a \cos(2\pi i/n) + cb \sin(2\pi i/n) $ where $c \in \mathbb{R}$ is a trainable weight.  
This amounts to mapping the irreducible $\rho_1$ representation of $C_n$ to the regular representation.  This process can also be extended to lifting velocity fields to features $f \colon C_4 \ltimes (\mathbb{Z}^2,+) \to \mathbb{R}$ over the group of discrete rotations and translations.  As an additional advantage, the feature maps $f$ are compatible with element-wise non-linearity.  

  
\subsection{Relaxed Steerable Convolution}
Although relaxed group convolution does not require pre-computing an equivariant kernel basis, it is limited to discrete (or compact) groups and is inefficient when the group order is large. Thus, we also propose relaxed steerable convolutions. 

\paragraph{$\mathbf{G}$-Steerable 2D Convolution:} First, we explicitly write out the formula for $G$-steerable 2D convolution described by \eqref{steerable}. 
Let $\lbrace \Phi_l \rbrace_{l=1}^L$ be an equivariant kernel basis of $L$ non-trainable kernels that satisfy \eqref{steerable} for given input and output representations $\rho_{\mathrm{in}}$ and $\rho_{\mathrm{out}}$. Denote $K = \lbrace -k,\ldots,k \rbrace$.
Denote the input feature as $f_{\mathrm{in}} \colon \mathbb{Z}^2 \to \mathbb{R}^{c_\mathrm{in}}$,  predetermined equivariant kernels $\Phi_l \colon K^2 \to \mathbb{R}^{c_{\mathrm{out}}\times c_{\mathrm{in}}}$, and a trainable weight tensor $w \in \mathbb{R}^{c_{\mathrm{out}}\times c_{\mathrm{in}} \times L}$. Then a $G$-steerable convolution produces an output $f_{\mathrm{out}} = f_{\mathrm{in}} \star_{\mathbb{Z}^2} \Phi \colon \mathbb{Z}^2 \to \mathbb{R}^{c_\mathrm{out}}$ defined as
\begin{equation}
\resizebox{0.95\hsize}{!}{$f_{\mathrm{out}}(\mathbf{x}) = \sum_{\mathbf{y} \in \mathbb{Z}^2} \sum_{l=1}^L (w_l \odot \Phi_l(\mathbf{y})) f_{\mathrm{in}}(\mathbf{x} + \mathbf{y})$} \label{eqn:gsteerable}
\end{equation}
for a position  $\mathbf{x} \in \mathbb{Z}^2$  in the input. Here $\odot$ denotes element-wise product in $\mathbb{R}^{c_{\mathrm{out}}\times c_{\mathrm{in}}}$, and $\mathbf{y} \in \mathbb{Z}^2$ is a spatial location in the kernel.


\paragraph{Relaxed $G$-Steerable 2D Convolution:}  We relax \eqref{eqn:gsteerable} and break symmetry by introducing a weight $w$ that depends on $\mathbf{y}$. As $w$ is freely trainable, this breaks the strict positional dependence of $\Phi_l$ imposed by \eqref{steerable}.  
Formally, letting $w \colon K^2 \to \mathbb{R}^{c_{\mathrm{out}}\times c_{\mathrm{in}} \times L}$ be the weight, we define the relaxed steerable convolution $f_{\mathrm{out}} = f_{\mathrm{in}} \hat{\star}_{\mathbb{Z}^2} \Phi$ by 
\begin{equation}
\resizebox{0.95\hsize}{!}{$f_{\mathrm{out}}(\mathbf{x}) = \sum_{\mathbf{y} \in \mathbb{Z}^2} \sum_{l=1}^L (w_l(\mathbf{y}) \odot \Phi_l(\mathbf{y})) f_{\mathrm{in}}(\mathbf{x} + \mathbf{y})$}. \label{eqn:relaxed_steerable}
\end{equation}


When $G$ is a rotation group and $k > 0$, we can define $w_l(\mathbf{y}) = w_l(\theta)$, where $\theta = \mathrm{arctan2}(\mathbf{y})$.  Since the weight depends only on the angle of the vector $\mathbf{y}$, we use fewer parameters. To prevent the model from becoming overly relaxed, we initialize $w_l(\mathbf{y})$ equally for every $\mathbf{y}$ and penalize the value differences in $w_l(\mathbf{y})$ during training, which we  describe in Section \ref{soft_equ_reg}. 

For the translation group, we can relax the steerable convolution by further allowing $w \colon \mathbb{Z}^2 \times K^2 \to \mathbb{R}^{c_{\mathrm{out}}\times c_{\mathrm{in}} \times L}$ to vary with the input position $\mathbf{x}$ as well:
\begin{equation}
\resizebox{0.95\hsize}{!}{$f_{\mathrm{out}}(\mathbf{x}) = \sum_{\mathbf{y} \in \mathbb{Z}^2} \sum_{l=1}^L (w_l(\mathbf{x},\mathbf{y}) \odot \Phi_l(\mathbf{y})) f_{\mathrm{in}}(\mathbf{x} + \mathbf{y})$} .\label{equ:combined_relaxed}
\end{equation}
However, the above equation is impractical as  the space of the trainable weight is too large.  We propose using a low-rank factorization of $w$ to reduce dimensionality,
\[
    w_l(\mathbf{x},\mathbf{y}) = \sum_{r=1}^R a_r(\mathbf{x}) b_{r,l}(\mathbf{y})  
\]
where $a_r \colon \mathbb{Z}^2 \to \mathbb{R}$ and $b_{r,l}: K^2 \to \colon \mathbb{R}^{c_\mathrm{out} \times c_\mathrm{in}}$. Then \eqref{equ:combined_relaxed} becomes a combination of relaxed translation group convolution and relaxed steerable convolution.

\subsection{Soft Equivariance Regularization}\label{soft_equ_reg}
To encourage equivariance and prevent the model from becoming over-relaxed, we add regularization terms to the loss function on the symmetry-independent weights $w$ during training. For relaxed group convolution, we add the following regularizer  to constrain  $w$ in \eqref{equ:relaxed_group_conv},
\[
    L_{\mathrm{gconv}} = \alpha\sum_{i=1}^L\sum_{g,h \in G}\|w_i(h) - w_i(g)\|.
\]
For  relaxed steerable convolution, we impose the following loss term to prevent the $w \colon K^2 \to \mathbb{R}^{c_{\mathrm{out}}\times c_{\mathrm{in}} \times L}$ in \eqref{eqn:relaxed_steerable} from varying too much across the kernel spatial domains,
\[
    L_{\mathrm{sconv}} = \alpha\left(\left\|\frac{\partial w(m,n)}{\partial m}\right\| + \left\|\frac{\partial w(m,n)}{\partial n}\right\|\right).
\]
The hyperparameter $\alpha$ does not directly control how equivariant the model is, it only a places a equivariance prior on the model to be as equivariant as possible given the data.

\subsection{Other Alternatives for  Approximate Symmetry}\label{other_app}
We also explored  three alternative ways of building approximately equivariant models. 

\textbf{Lift expansion.} \citet{wang2021equivariant} proposed \texttt{Lift Expansion} for modeling partial symmetry, in which the input space can be factorized into an equivariant subspace and a non-equivariant subspace. The model uses a non-equivariant encoder that is tiled across the equivariant dimensions of the feature map as additional channels in equivariant neural nets. This method can  also model approximate symmetry when both the non-equivariant encoder and the main equivariant backbone are fed with the same input. The encoder can extract non-equivariant features that are then treated as having a trivial representation type and included in the main equivariant model to break perfect equivariance.  Note that while treating the output data as having a trivial representation type enforces invariance, treating the input data this way imposes no constraints. 

\textbf{Constrained locally connected neural nets (CLCNN).}
Another way of building approximately equivariant models is using a very flexible model while imposing soft equivariance constraints on the kernels. We use a locally connected neural network that has the same locality property as convolution but the weights are not shared across the spatial domain \cite{wadekar_2019}. Thus, it does not have translation equivariance and employs many more parameters than convolution. Suppose $\Phi \colon \mathbb{Z}^2 \times K^2 \to \mathbb{R}^{c_\mathrm{out} \times c_\mathrm{in}}$ is the filter bank. In addition to the prediction loss, we use the equivariant kernel constraint \eqref{steerable} in the objective with a hinge loss instead of solving the constraints explicitly before training:
\[
    L_{\mathrm{hinge}} = \alpha\sum_{h \in G}\|\rho_{\text{out}}(h) \Phi(x) \rho_{\text{in}}(h^{-1}) - \Phi(hx)\|.
\]
\textbf{Combination of non-equivariant and equivariant layers.}
We also build models that begin with non-equivariant layers followed by equivariant layers. The early layers of the model  map observations with approximate symmetries to a space with an explicit symmetry actions.

\subsection{Equivariant Error Analysis}

Our hypothesis is that if the ground truth function $f$ is approximately equivariant, then a model class with a similar degree of approximate equivariance would better approximate $f$ than a strictly equivariant class or a class without bias towards symmetry.  

We define equivariance error, which quantifies how much a function $f$ is approximately equivariant.

\begin{definition}[Equivariance Error]
Let $f \colon X \to Y$ be a function and $G$ be a group. Assume that $G$ acts on $X$ and $Y$ via representation $\rho_{X}$ and $\rho_{Y}$. Then the \textbf{equivariance error} of $f$ is 
\[
\|f\|_{EE} = \sup_{x,g}
\|f(\rho_{X}(g)(x)) - \rho_{Y}(g)f(x)\|.
\]
\end{definition}

That is, $f$ is $\epsilon$-approximately equivariant if and only if $\| f\|_\mathrm{EE} < \epsilon.$ 

We note that a strictly equivariant model  cannot perfectly learn an approximately equivariant  function. As stated by the following proposition, such a model would make errors at least proportional to the equivariance error. This motivates our choice to use the model class of approximately equivariant networks.

\begin{proposition}
Let $f \colon X \to Y$ where $G$ acts on $X$ and $Y$ by $\rho_X$ and $\rho_Y$ which are norm-preserving.   Assume $f$ is approximately equivariant with $\|f\|_{\mathrm{EE}} \geq 0$. Assume $f_\theta$ is a $G$-equivariant approximator for $f$.  Then there exists $x_0 \in X$ such that
\[
\| f(x_0) - f_\theta(x_0) \| \geq \|f\|_{\mathrm{EE}} /2.
\]
\end{proposition}
For simplicity, we assume representations $\rho$ which are norm preserving, as with rotations, reflections, and permutations, although this assumption can be removed by inserting a factor to account for the operator norm $\| g \|$.

By similar logic, we can also show that given a model class which contains $\epsilon$-approximately equivariant functions for varying $\epsilon$, the equivariance error of the approximator will converge to the equivariance error of the ground truth function as they converge in model error. Although the supremum norm is used for equivariance and model error, the result holds for other norms as well. 

\begin{proposition}
Let $\lbrace f_\theta \rbrace$ be an approximately equivariant model class with varying $\| f_\theta \|_\mathrm{EE} \in \mathbb{R}_{\geq 0}$.  Assume a $G$-invariant norm.  Let $f \colon X \to Y$ be a function with $\| f \|_{\mathrm{EE}} = \epsilon$.  Assume $\| f - f_\theta \|_\infty \leq c$.  Then $| \| f \|_\mathrm{EE} - \| f_\theta \|_\mathrm{EE} | \leq 2c + \epsilon.$ 
\end{proposition}

The proofs can be found in Appendix \ref{equ_error}.  

\section{Related Work}
\subsection{Equivariance and Invariance}
Symmetry  has long been implicitly used in DL to design networks with known invariances and equivariances. Convolutional neural networks enabled breakthroughs in computer vision by leveraging translational equivariance \cite{zhang1988shift, lecun1989backpropagation, zhang1990parallel}. Similarly, recurrent neural networks \cite{rumelhart1986learning, hochreiter1997long}, graph neural networks \cite{Maron2019InvariantAE, Garcia2021EN}, and capsule networks \cite{sabour2017dynamic, hinton2011transforming} all impose symmetries.
Equivariant DL models have achieved remarkable success in learning image data \citep{cohen2019gauge,  weiler2019e2cnn, cohen2016group, chidester2018rotation, Lenc2015understanding, kondor2018generalization, bao2019equivariant, worrall2017harmonic, cohen2016steerable, finzi2020generalizing, weiler2018learning, dieleman2016cyclic, Ghosh19Scale, Sosnovik2020Scale-Equivariant}. 

There is also a deep connection between symmetries and physics. Noether’s law gives a correspondence between conserved quantities and groups of symmetries. Thus, the study of equivariant nets in learning dynamical systems has gained popularity. 
\citet{Walters2021ECCO} proposed a rotationally-equivariant continuous convolution model for improved pedestrian and vehicle trajectory predictions. 
\citet{pmlr-v139-holderrieth21a} introduced Steerable Conditional Neural Processes for learning stochastic processes in physics that have invariances and equivariances.
\citet{wang2020incorporating} designed fully equivariant models with respect to symmetries of scaling, rotation, and uniform motion in physical dynamics. 

But most dynamics in real world do not have perfect symmetry and thus the proposed models might be overly-constrained. Recently, some work explored the idea of building approximately equivariant networks \cite{van2022relaxing, romero2021learning}.
\citet{Elsayed2020Revisiting} showed that spatial invariance may be overly restrictive and relaxing the spatial weight sharing could outperform both convolution and local connectivity. 
\citet{finzi2021residual} proposed a mechanism that sums equivariant and non-equivariant MLP layers for modeling soft equivariances, but it cannot handle large data like images or high-dimensional physical dynamics due to the number of weights in the fully connected layers.  Our method in contrast has more efficient convolutional layers and uses relaxed constraints to achieve approximate equivariance.

\subsection{Learning Dynamical Systems}
There is an increasing number of works in modeling dynamical systems with deep learning \citep{Shi2017DeepLF,chen2018neural,manek2020learning, Azencot2020ForecastingSD, xie2018tempogan, eulerian, pfaff2021learning}. An essential topic is physics-guided deep learning \citep{raissi2017physics, lutter2018deep, PDE-CDNN, li2020fourier, Wang2020TF} which integrates inductive biases from physical systems to improve learning. For example, 
\citet{Wang2020TF} proposed a hybrid model marrying the RANS-LES coupling method and the custom-designed U-net. \citet{Greydanus2019HamiltonianNN} and \citet{Cranmer2020LagrangianNN} build models on Hamiltonian and Lagrangian mechanics that respect conservation laws. \citet{Guen2021AugmentingPM} proposed a framework that augments physics-based models with deep data-driven models for forecasting dynamical systems.
In this work, we encode approximate symmetries as inductive biases into DL models to improve dynamics prediction without over-constraining the representation power.

\section{Experiments}

\newcolumntype{P}[1]{>{\centering\arraybackslash}p{#1}}
\begin{table*}[ht!]
\small
    \centering
    \caption{Prediction RMSE on three synthetic smoke plume datasets with approximate symmetries. Our proposed \texttt{RGroup} and \texttt{RSteer} methods demonstrate competitive performance. \textit{Future} means  testing data lies in the future time of the training data. \textit{Domain} means training and test data are from different spatial domain. }
    \label{tab:results_smoke}
    \begin{tabular}{P{1.3cm}|P{1cm}|P{1cm}P{1cm}P{1cm}P{1cm}P{1cm}P{1cm}P{1cm}P{1.1cm}P{1.4cm}}\toprule
       \multicolumn{2}{c}{Model}  &  $\texttt{MLP}$ &  $\texttt{Conv}$ &  $\texttt{Equiv}$ &  $\texttt{Rpp}$ &  $\texttt{Combo}$ &  $\texttt{CLCNN}$ &  $\texttt{Lift}$ &  $\texttt{RGroup}$ &  $\texttt{RSteer}$ \\
        \midrule
        \multirow{2}{*}{Translation} & \textrm{Future}   & 1.56\scriptsize{$\pm$0.08} & ----- &  0.94\scriptsize{$\pm$0.02}  & 0.92\scriptsize{$\pm$0.01} &  1.02\scriptsize{$\pm$0.02} & 0.92\scriptsize{$\pm$0.01} & 0.87\scriptsize{$\pm$0.03}  & \textbf{0.71\scriptsize{$\pm$0.01}}  & ----- \\
        & \textrm{Domain}   & 1.79\scriptsize{$\pm$0.13} & ----- &  0.68\scriptsize{$\pm$0.05}  & 0.93\scriptsize{$\pm$0.01} & 0.98\scriptsize{$\pm$0.01}  & 0.89\scriptsize{$\pm$0.01} & 0.70\scriptsize{$\pm$0.00}  & \textbf{0.62\scriptsize{$\pm$0.02}}  & ----- \\
        \midrule
        \multirow{2}{*}{Rotation} & \textrm{Future}   &   1.38\scriptsize{$\pm$0.06} &1.21\scriptsize{$\pm$0.01} &1.05\scriptsize{$\pm$0.06}  &0.96\scriptsize{$\pm$0.10} &1.07\scriptsize{$\pm$0.00} &0.96\scriptsize{$\pm$0.05} &0.82\scriptsize{$\pm$0.08} &0.82\scriptsize{$\pm$0.01} &\textbf{0.80\scriptsize{$\pm$0.00}}\\
        & \textrm{Domain}  &  1.34\scriptsize{$\pm$0.03}&1.10\scriptsize{$\pm$0.05}&0.76\scriptsize{$\pm$0.02}&0.83\scriptsize{$\pm$0.01}&0.82\scriptsize{$\pm$0.02}&0.84\scriptsize{$\pm$0.10}&0.68\scriptsize{$\pm$0.09}&0.73\scriptsize{$\pm$0.02}&\textbf{0.67\scriptsize{$\pm$0.01}}\\
        \midrule
        \multirow{2}{*}{Scaling} & \textrm{Future}   &  2.40\scriptsize{$\pm$0.02} & 0.83\scriptsize{$\pm$0.01} &  0.75\scriptsize{$\pm$0.03} & 0.81\scriptsize{$\pm$0.09}  &  0.78\scriptsize{$\pm$0.04} & 1.03\scriptsize{$\pm$0.01}  & 0.85\scriptsize{$\pm$0.01} & 0.76\scriptsize{$\pm$0.04} & \textbf{0.70\scriptsize{$\pm$0.01}}\\
        & \textrm{Domain}   & 1.81\scriptsize{$\pm$0.18} & 0.95\scriptsize{$\pm$0.02} & 0.87\scriptsize{$\pm$0.02}  & 0.86\scriptsize{$\pm$0.05}  & 0.85\scriptsize{$\pm$0.01} & 0.83\scriptsize{$\pm$0.05}  & 0.77\scriptsize{$\pm$0.02} & 0.86\scriptsize{$\pm$0.12}  & \textbf{0.73\scriptsize{$\pm$0.01}} \\
        \bottomrule
    \end{tabular}
    \vskip -0.1in
\end{table*}

\paragraph{Baselines}
We compare with several state-of-the-art (SoTA) methods from those without symmetry bias to perfect symmetry and SoTA approximately symmetric models. 
\begin{itemize}[noitemsep,topsep=2pt,parsep=2pt,partopsep=2pt,leftmargin=*]
    \item \texttt{MLP}: multi-layer perceptrons, an non-equivariant baseline with a weaker inductive bias than convolution neural nets. 
    \item \texttt{ConvNet}: standard convolutional neural nets that have full translation symmetry.
    \item \texttt{Equiv}: fully equivariant convolutional models. It is same as \texttt{ConvNet} for translation symmetry. We use \texttt{E2CNN} \citep{e2cnn} for rotation and \texttt{SESN} \citep{Sosnovik2020Scale} for scaling symmetry. 
    \item \texttt{Rpp} \citep{finzi2021residual}: Residual Pathway Priors, a SOTA approximate equivariance model that sums up the outputs from equivariant and non-equvariant layers while posing constraints on the non-equivariant layer in the loss function. We use the combination of \texttt{MLP} and \texttt{ConvNet} for translation, \texttt{ConvNet} and \texttt{E2CNN} for rotation, and \texttt{ConvNet} and \texttt{SESN} for scaling.
    \item \texttt{Combo}: models that start with non-equivariant layers followed by equivariant layers, discussed in Section \ref{other_app}.
    \item \texttt{CLCNN}: locally connected neural networks with equivariance constraints imposed in the loss function.
    \item \texttt{Lift} \citep{wang2021equivariant}: Lift expansion for modeling partial symmetry. Both the encoder and the main equivariant backbone are fed with the same input.
\end{itemize} 
EMLP \cite{finzi2021emlp} is also a SoTA equivariant model, but it cannot handle large data like images as stated in the paper, so we do not include it as a baseline. 

  
\paragraph{Experiments Setup}
All models are trained to perform forward prediction of raw velocity fields given historical data. For all datasets, we use a sliding window approach to generate sequence samples. We perform a grid hyperparameter search as shown in Table \ref{tab:hyperparameter}, including learning rate, batch size, hidden dimension, number of layers,  number of prediction errors steps for training. We also tune the number of filter banks for group convolution-based models and the coefficient of weight constraints for relaxed weight-sharing models.  The input length is fixed as 10. Meanwhile, we make sure that the total number of trainable parameters for every model is less than $10^7$ for a fair comparison.

We test all models under two scenarios. For \textit{test-future}, we train and test on the same tasks but in different time steps. For \textit{test-domain}, we train and test on different simulations/regions with an 80\%-20\% split. All models are trained to make the prediction of the next step given the previous steps as input. The first scenario evaluates
how well the models can extrapolate into the future for the same task. The second scenario estimates the capability of the models to generalize across different simulations/regions. We forecast in an autoregressive manner to generate multi-step  predictions during inference and evaluate them based on 20-step prediction RMSEs. All results are averaged over 3 runs with random initialization.

\begin{figure*}[h!]
	\centering
	\includegraphics[width=0.48\textwidth]{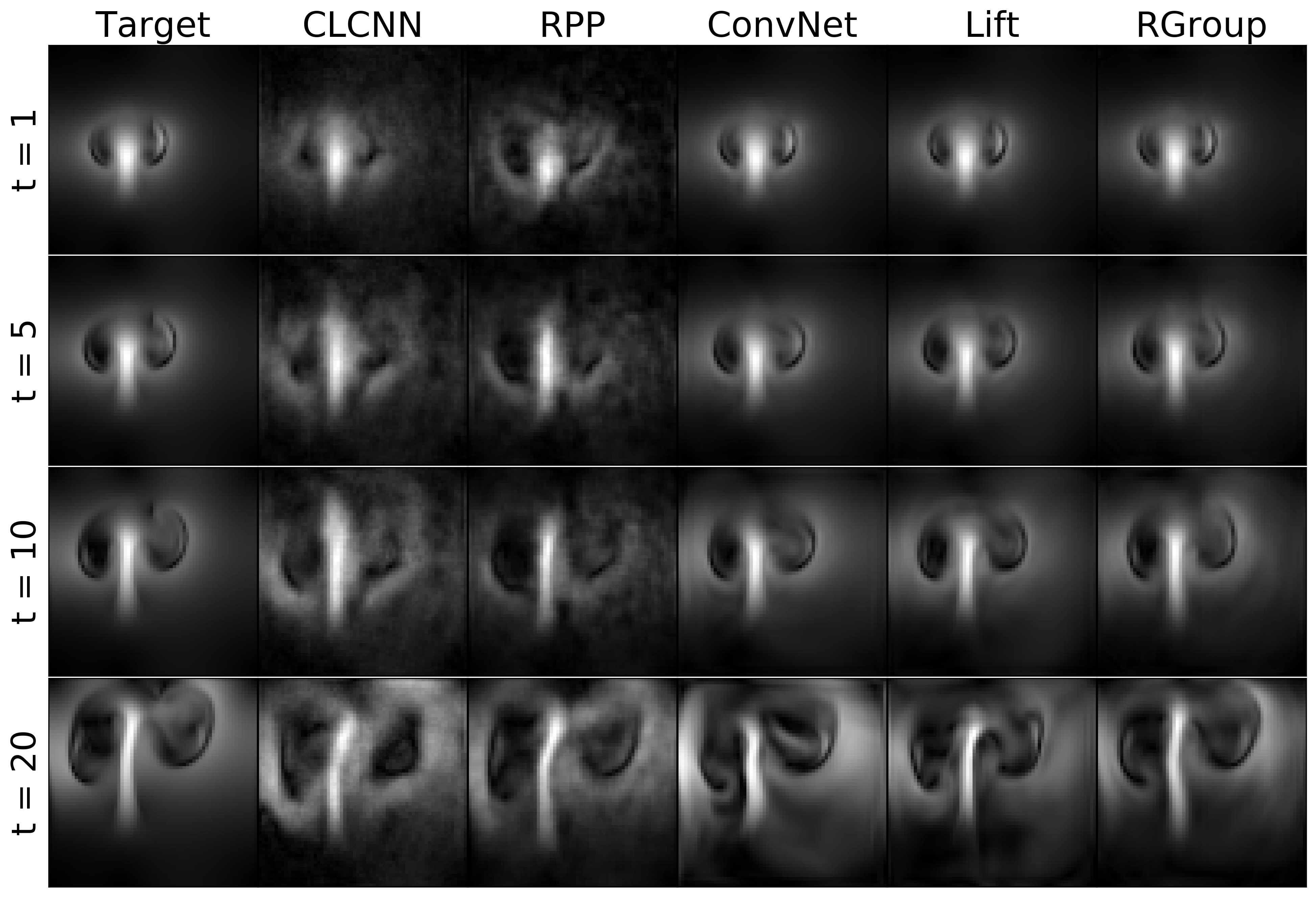} \hspace{1mm}
	\includegraphics[width=0.48\textwidth]{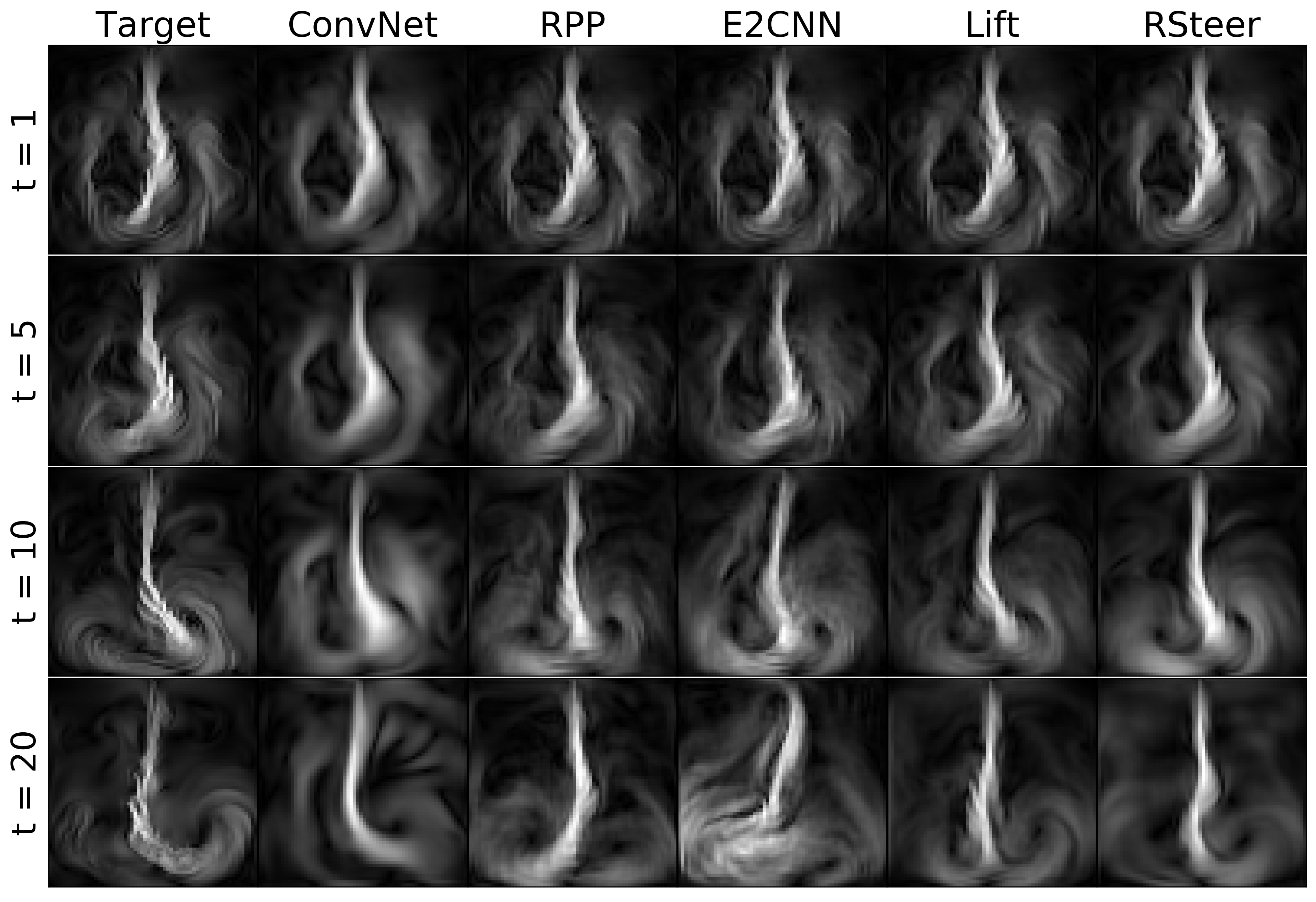}
	\caption{Target (ground truth) and model predictions comparison at time step 1, 5, 10, 20 for smoke simulation with approximate translation (left) and rotation (right) symmetries.}
	\label{vis:smoke}
\end{figure*}

\subsection{Experiments on Synthetic Smoke Plumes}
\paragraph{Data Description: }
The synthetic 64$\times$64 2-D smoke datasets are generated by PhiFlow \cite{phiflow} and contain smoke simulations with different initial conditions and external forces. We explore three symmetry groups: 1) \textit{Translation}: 35 smoke simulations with different inflow positions. We also horizontally split the entire domain into two separate sub-domains that have different buoyant forces. Although the inflow positions are translation equivariant, the closed boundary and the two different buoyant forces would break the equivariance. 2) \textit{Rotation}: 40 simulations with different inflow positions and buoyant forces. Both the inflow location and the direction of the buoyant forces have a perfect rotation symmetry with respect to $C_4$ group, but the buoyancy factor varies with the inflow positions to break the rotation symmetry. 3) \textit{Scaling}: It contains 40 simulations generated with different spatial steps $\Delta x$ and temporal steps $\Delta t$. And the buoyant force varies across the simulations to break the scaling symmetry. 

\paragraph{Prediction Performance: }
Table \ref{tab:results_smoke} shows the prediction RMSEs in three synthetic smoke plume datasets with different approximate symmetries by our proposed models and baselines.  CNNs are translation-equivariant because CNNs are inherently group convolution, where the group is the translation group, so we do not have a relaxed steerable model for translation. We can see, on the approximate translation dataset, our relaxed group convolution (\texttt{RGroup}) significantly outperforms baselines on both test sets. And for rotation and scaling,  the proposed relaxed steerable convolution (\texttt{RSteer}) always achieves the lowest RMSE and \texttt{RGroup} can outperform most baselines. 

Figure \ref{vis:smoke} shows the target and predictions of our proposed models and the best baselines at time step 1, 5, 10, 20 for smoke simulation with approximate translation (left) and rotation (right) symmetries. From the shape and frequency of the flows, predictions from our approximately equivariant models are much closer to the target than the baselines. Moreover, we can see that \texttt{E2CNN} predicts the smoke flowing to the wrong direction at time step 20, which could be a consequence of over-constraining from equivariance. 

Figure \ref{vis:alpha} in Appendix \ref{add_figs} shows the prediction performance of a scaling \texttt{RSteer} model trained with different regularization parameter $\alpha$ discussed in the section \ref{soft_equ_reg}. We see that the soft equivariance regularization can further improve its prediction performance on both test sets but large $\alpha$ may also hinder its learning.

\subsection{Learning Different Levels of Equivariance}
We use PhiFlow \cite{phiflow} to create 10 small smoke plume datasets with different levels of rotational equivariance. In each data set, both the inflow location and the direction of the buoyant forces have a perfect rotation symmetry with respect to the $C_4$ group. By varying the amount of difference in buoyant force between simulations with different inflow positions, we can control the amount of equivariance error in the data. The data equivariance error of each dataset is the mean absolute error between the simulations after they are all rotated back to the same inflow position. 

\begin{figure}[htb!]
	\centering
	\hspace*{-0.2cm} 
	\includegraphics[width=0.45\textwidth]{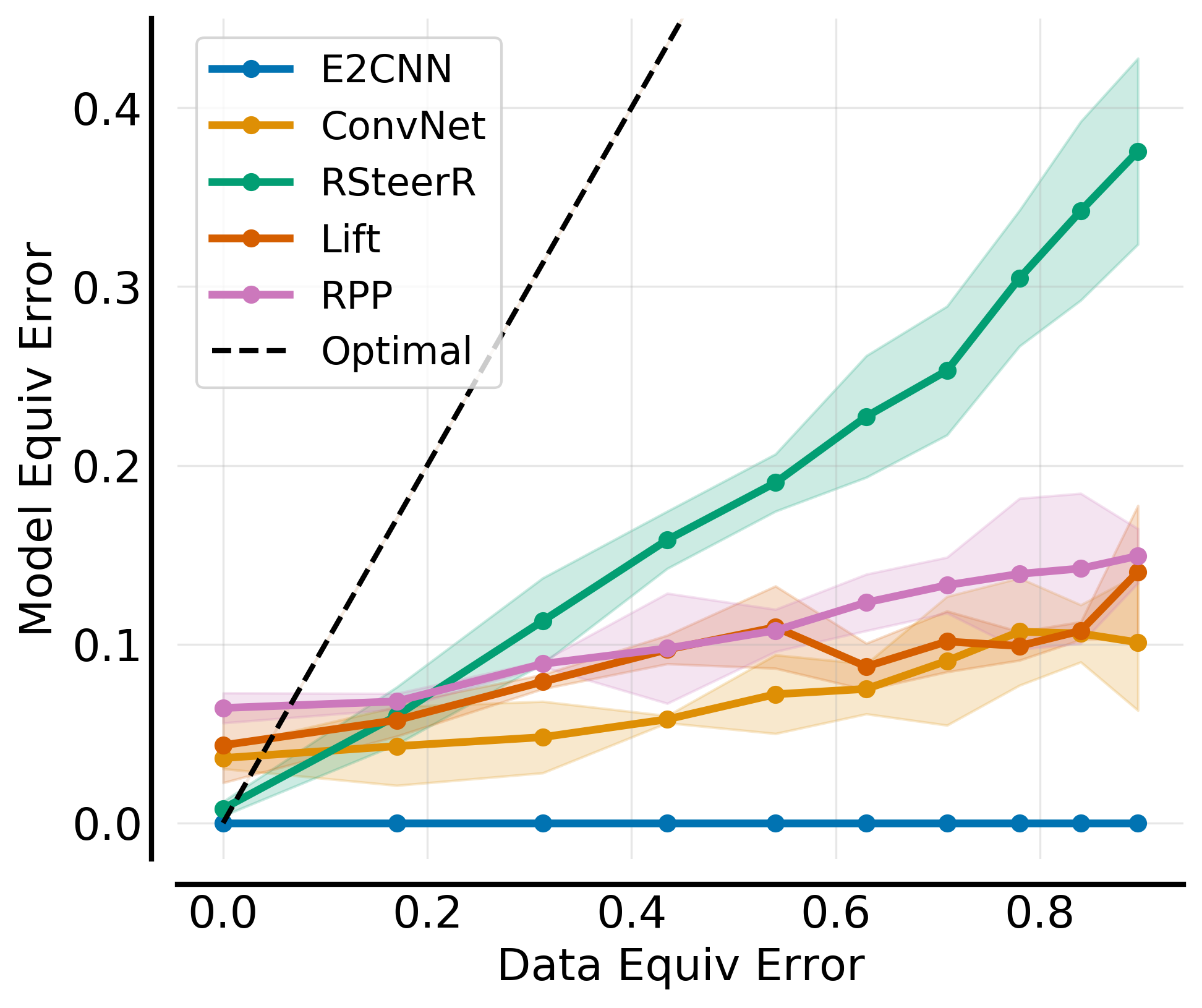}
	\caption{Model equivariance errors vs. data equivariance errors on synthetic smoke plume with different levels of rotational equivariance. We see that our \texttt{RSteer} can learn different levels of equivariance in the data much more accurately than other baselines.}
	\label{vis:equiv_error}
\end{figure}

We trained two-layer \texttt{ConvNet}, \texttt{E2CNN} and our relaxed rotation equivariant steerable convolution \texttt{RSteerR} on these 10 datasets. We calculate the equivariance error of each well-trained model based on Definition \ref{def_app}, where $G=C_4$ and the norm is L1 norm. From Figure \ref{vis:equiv_error}, we see that \texttt{E2CNN} always has zero equivariance error due to the overly restrictive symmetry constraint even if the data does not have perfect symmetry. And our \texttt{RSteerR} can learn different levels of equivariance in the data more accurately than other baselines. Since the prediction errors are not zeros, the equivariance errors in the model and data are not the same. This experiment demonstrates that our proposed methods based on relaxed weight sharing can learn the \textit{correct} amount of inductive biases from data while avoiding the stringent symmetry constraints.

\begin{table*}[t!]
\small
    \centering
    \caption{Prediction RMSEs  on  experimental jet flow data for different models. The proposed \texttt{RSteer} and \texttt{RSteer}  are designed for the corresponding assumed symmetry group. $\texttt{RSteerTR}$ and $\texttt{RSteerTS}$ combines relaxed translation group convolution with relaxed rotation and scale steerable convolution. }
    \label{tab:results_jet}
    \begin{tabular}{P{0.8cm}|P{0.9cm}P{0.9cm}P{1.15cm}|P{0.9cm}P{0.9cm}P{1.15cm}|P{0.9cm}P{0.9cm}P{1.15cm}|P{1.25cm}P{1.25cm}}\toprule
       Model  &  $\texttt{Conv}$ &  $\texttt{Lift}$ &  $\texttt{RGroup}$ & $\texttt{E2CNN}$ &  $\texttt{Lift}$ & $\texttt{RSteer}$ &  $\texttt{SESN}$  & $\texttt{Rpp}$   & $\texttt{RSteer}$ &  $\texttt{RSteerTR}$ & $\texttt{RSteerTS}$\\
        \midrule
   & \multicolumn{3}{c|}{Translation} & \multicolumn{3}{c|}{Rotation} &\multicolumn{3}{c|}{Scaling} & \multicolumn{2}{c}{Combination}   \\    
   \midrule
        \textrm{Future}  & 0.22\scriptsize{$\pm$0.06} & 0.17\scriptsize{$\pm$0.02} & 0.15\scriptsize{$\pm$0.00} &  0.21\scriptsize{$\pm$0.02}  & 0.18\scriptsize{$\pm$0.02}  & 0.17\scriptsize{$\pm$0.01}  &  0.15\scriptsize{$\pm$0.00}  &  0.16\scriptsize{$\pm$0.06}& 0.14\scriptsize{$\pm$0.01} & 0.14\scriptsize{$\pm$0.01} & 0.14\scriptsize{$\pm$0.02} \\
        \textrm{Domain}  & 0.23\scriptsize{$\pm$0.06} & 0.18\scriptsize{$\pm$0.02} & 0.16\scriptsize{$\pm$0.01} &  0.27\scriptsize{$\pm$0.03}  &  0.21\scriptsize{$\pm$0.04} & 0.16\scriptsize{$\pm$0.01} &  0.16\scriptsize{$\pm$0.01} & 0.16\scriptsize{$\pm$0.07} & 0.15\scriptsize{$\pm$0.00} & 0.15\scriptsize{$\pm$0.01} & 0.15\scriptsize{$\pm$0.00} \\
        \bottomrule
    \end{tabular}
\end{table*}

\begin{figure}[h!]
	\centering
	\hspace*{-0.1cm} 
	\includegraphics[width=0.49\textwidth]{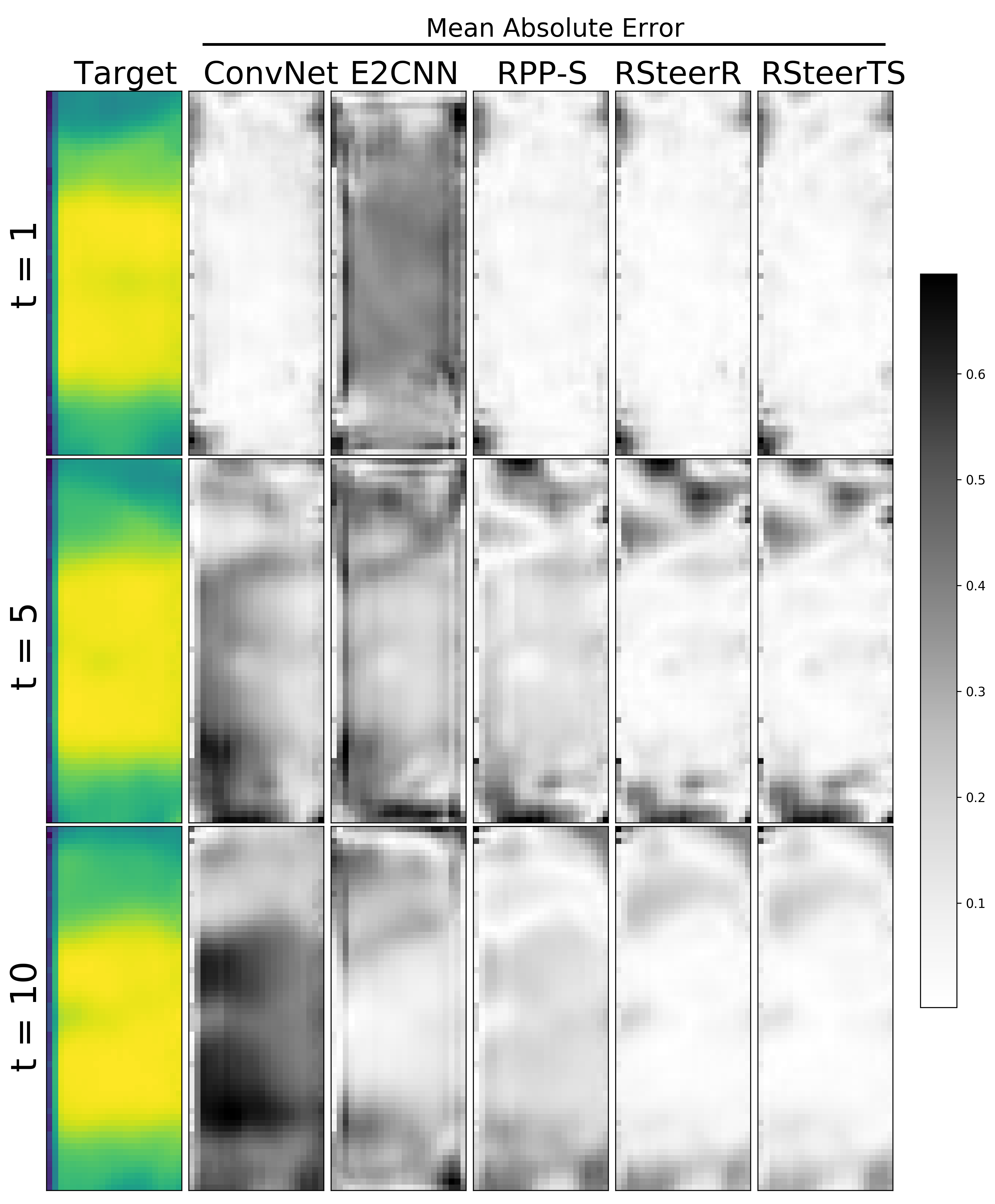}
	\caption{Target jet flow velocity norm fields and the prediction errors (MAE) of different models over $10$ time steps.}
	\label{vis:jet_flow_pred}
\end{figure}

\subsection{Experiments on Experimental Jet Flow Data}
\textbf{Data Description.} We use real experimental data on 2D turbulent velocity in NASA multi-stream jets that are measured using time-resolved particle image velocimetry \cite{Bridges2017MeasurementsOT}. Figure \ref{vis:jet_flow} in the Appendix \ref{add_figs} visualizes the measurement system of the jet flow. The white boxes show fields of view acquired on the streamwise plane at the jet centerline for multi-stream flows. There are three vertical stations at each axial location/white box, as illustrated by the pink lines. In other words, the dataset was acquired by 24 different stations at different locations. Since the data collected at the different locations are not acquired concurrently, we do not have the complete velocity fields of entire jet flows at each time step. Thus, we trained and test models on 24 62$\times$23 sub-regions of jet flows.

\textbf{Prediction Performance.}
We compare three equivariant models, three best-performing approximately equivariant baselines in the previous experiment as well as our proposed relaxed steerable convolution and relaxed group convolution. Table \ref{tab:results_jet} shows the prediction RMSEs on the jet flow dataset, and we group the results by each symmetry in the table. For each symmetry, our models based on relaxed weight sharing achieve lower errors than not only the fully equivariant model but also approximately equivariant baselines.  We also experimented with combining relaxed translation group convolution with relaxed rotation and scale steerable convolution, which correspond to $\texttt{RSteerTR}$ and $\texttt{RSteerTS}$ respectively in the table. We observe that $\texttt{RSteerTR}$ outperform  both $\texttt{RGroup}$ with relaxed  group convolution and $\texttt{RSteer}$ with relaxed steerable convolution. This implies relaxing more than one equivariance constraint can potentially lead to even better performance. Figure \ref{vis:jet_flow_pred} visualizes the target and mean absolute errors between model jet flow predictions and the ground truth (target), and we can see that our relaxed steerable CNNs achieve the lowest errors.

We also performed experiments on real-world ocean dynamics. We observe that all models have very close prediction performance after fine-tuning. Unlike the smoke plume or jet flow experiments in which our model's approximate equivariance bias better matched the ground truth than either the strictly equivariant model or the non-equivariant model, in this case all three levels of equivariance bias perform similarly.  We hypothesize that, while strict rotational symmetry is a feature of ocean currents, imposing it as a strict inductive bias does not provide a significant advantage over the baseline CNN.  Therefore, imposing a soft approximate equivariance bias also does not provide an advantage.   For additional results, see Appendix \ref{app:ocean}.

\section{Discussion}

We propose a new class of approximately equivariant networks that avoid stringent symmetry constraints to better fit real-world scenarios. Our methods  strike a good balance between inductive biases and model flexibility by relaxing the weight-sharing and weight-tying schemes in group convolution and steerable convolution. Based on the experiments on smoke plume simulations and real-world jet flow data, we observe that our proposed approximate equivariant networks  can outperform many state of the art baselines with no symmetry bias or with overly strict symmetry constraints.  Future work includes applying our relaxed weight sharing design to graph neural networks and theoretical analysis for approximately equivariant networks, including universal approximation and generalization.

\section*{Acknowledgements}
This work was supported in part by U.S. Department Of Energy, Office of Science, U. S. Army Research Office
under Grant W911NF-20-1-0334, Google Faculty Award, and NSF Grant \#2134274. R. Walters is supported by a
Postdoctoral Fellowship from the Roux Institute and NSF grants \#2107256 and \#2134178. We also thank Nicholas
J. Georgiadis and Mark P. Wernet and the NASA Glenn Research Center for providing the jet flow data and the instructions. This research used resources
of the National Energy Research Scientific Computing Center, a DOE Office of Science User Facility supported by
the Office of Science of the U.S. Department of Energy under Contract No. DE-AC02-05CH11231.



\bibliography{references}
\bibliographystyle{icml2022}

\newpage
\appendix
\onecolumn
\section{Experiments Details}
\subsection{Hyperparameter Tuning}
We perform grid hyperparameters search as shown in Table \ref{tab:hyperparameter}, including learning rate, batch size, hidden dimension, number of layers,  number of steps of prediction errors for training. We also tune the number of filter banks for group convolution models and the coefficient of weight constraints $\alpha$ for relaxed weight sharing models. The input length is fixed as 10.
In the meanwhile, we make sure the total number of trainable parameters for every model is fewer than $10^7$ in order to make fair comparison.
\begin{table*}[ht!]
\small
    \centering
    \caption{Hyperparameter Tuninig Range.}
    \label{tab:hyperparameter}
    \begin{tabular}{P{1.8cm}|P{1.8cm}|P{1.8cm}|P{1.8cm}|P{1.8cm}|P{2.7cm}|P{2.5cm}}\toprule
         LR  & Batch size & Hid-dim &  Num-layers &  Num-banks & \#Steps for Backprop & $\alpha$\\
        \midrule
        $10^{-2} \sim 10^{-5}$  & $8 \sim 64$ & $64 \sim 512$ &  $3 \sim 6$ &  $1 \sim 4$ & $3 \sim 6$ & $0, 10^{-2}, 10^{-4}, 10^{-6}$\\
        \bottomrule
    \end{tabular}
\end{table*}

\subsection{Additional Experiments on Real-world Ocean Dynamics}\label{app:ocean}

\begin{table*}[ht!]
\small
    \centering
    \caption{Prediction RMSE on ocean currents data.}
    \label{tab:results_ocean}
    \begin{tabular}{P{0.8cm}|P{0.9cm}P{0.9cm}P{1.15cm}|P{0.9cm}P{0.9cm}P{1.15cm}|P{0.9cm}P{0.9cm}P{1.15cm}|P{1.25cm}P{1.25cm}}\toprule
       Model  &  $\texttt{Conv}$ &  $\texttt{LiftT}$ &  $\texttt{RGroupT}$ & $\texttt{E2CNN}$ &  $\texttt{LiftR}$ & $\texttt{RSteerR}$ &  $\texttt{SESN}$  & $\texttt{RppS}$   & $\texttt{RSteerS}$ &  $\texttt{RSteerTR}$ & $\texttt{RSteerTS}$\\
        \midrule
        \textrm{Future}  & 0.52\scriptsize{$\pm$0.02}  & 0.52\scriptsize{$\pm$0.01} & 0.51\scriptsize{$\pm$0.00}  &  0.51\scriptsize{$\pm$0.03}  & 0.56\scriptsize{$\pm$0.06}   & 0.51\scriptsize{$\pm$0.01}  & 0.51\scriptsize{$\pm$0.01} &  0.50\scriptsize{$\pm$0.01}& 0.50\scriptsize{$\pm$0.01} &  0.50\scriptsize{$\pm$0.01}& 0.50\scriptsize{$\pm$0.02} \\
        \textrm{Domain}  & 0.46\scriptsize{$\pm$0.02}  & 0.46\scriptsize{$\pm$0.01} & 0.45\scriptsize{$\pm$0.01}  &  0.45\scriptsize{$\pm$0.01} &  0.53\scriptsize{$\pm$0.04} & 0.45\scriptsize{$\pm$0.02}  & 0.45\scriptsize{$\pm$0.02}  &  0.42\scriptsize{$\pm$0.03}& 0.42\scriptsize{$\pm$0.01} & 0.42\scriptsize{$\pm$0.01} & 0.42\scriptsize{$\pm$0.01} \\
        \bottomrule
    \end{tabular}
\end{table*}

\paragraph{Data Description: }
We use the reanalysis ocean current velocity data generated by the NEMO ocean engine \cite{Madec2008NEMOOE}. We selected an area (-180 $\sim$ - 150, -30 $\sim$ 0)from the Pacific Ocean from 01/01/2021 to 12/31/2021 and extracted 36 64×64 sub-regions for our experiments. We not only test all models on the test sets with different time range and spatial domain from the training set.                          

\paragraph{Prediction Performance: }
We compare three equivariant models, three best approximately equivariant baselines as well as our proposed relaxed steerable CNNs and relaxed group convolutions.

\subsection{Equivariance Error Analysis} \label{equ_error}

\begin{proposition}
Let $f \colon X \to Y$ where $G$ acts on $X$ and $Y$ by $\rho_X$ and $\rho_Y$ which are norm-preserving.   Assume that $f$ is approximately equivariant with $\|f\|_{\mathrm{EE}} \geq 0$. Assume $f_\theta$ is a $G$-equivariant approximator for $f$.  Then there exists $x_0 \in X$ such that
\[
\| f(x_0) - f_\theta(x_0) \| \geq \|f\|_{\mathrm{EE}} /2.
\]
\end{proposition}

\begin{proof}
We leave implicit the action maps $\rho_X$ and $\rho_Y$.
By definition there exists $x \in X$ and $g \in G$ such that $\|f(g x) - g f(x)\| = \| f \|$, whereas $f_\theta(g x) -  g f_\theta(x) = 0.$  Thus by triangle inequality
\begin{align*}
&\| f \|_{\mathrm{EE}} = \| f(gx) - g f(x) \| \\
&= \| f(gx) - g f(x) - f_\theta(g x) +  g f_\theta(x) \| \\
&  \leq \| f(gx) - f_\theta(gx) \| + \| g f_\theta(x) - g f(x) \| 
\end{align*}
As the $G$-action is norm-preserving, 
\[
    \| f \| \leq \| f(gx) - f_\theta(gx) \| + \|  f_\theta(x) -  f(x) \|.
\]
Thus either $\| f(gx) - f_\theta(gx) \|$ or   $\|  f_\theta(x) -  f(x) \|$ is greater than $\|f\| /2$ in which case set $x_0$ to be $gx$ or $x$ respectively. 
\end{proof}

\begin{proposition}
Let $\lbrace f_\theta \rbrace$ be an approximately equivariant model class with varying $\| f_\theta \|_\mathrm{EE} \in \mathbb{R}_{\geq 0}$.  Assume a $G$-invariant norm.  Let $f \colon X \to Y$ be a function with $\| f \|_{\mathrm{EE}} = \epsilon$.  Assume $\| f - f_\theta \|_\infty \leq c$.  Then $\| \| f \|_\mathrm{EE} - \| f_\theta \|_\mathrm{EE} \| \leq 2c + \epsilon.$ 
\end{proposition}

\begin{proof}
By triangle inequality and invariance of the norm, 
\begin{align*}
    &\| g f_\theta(x) - f_\theta(gx) \| \leq  \| g f_\theta(x) - g f(x) \| \\
    &\ \  + \| g f(x) - f(gx) \| + \| f(gx) - f_\theta(gx) \| \\
    &\leq  2c+ \epsilon.
\end{align*} 
\end{proof}

\subsection{Additional Figures} \label{add_figs}
\begin{figure}[htb!]
	\centering
	\includegraphics[width=0.6\textwidth]{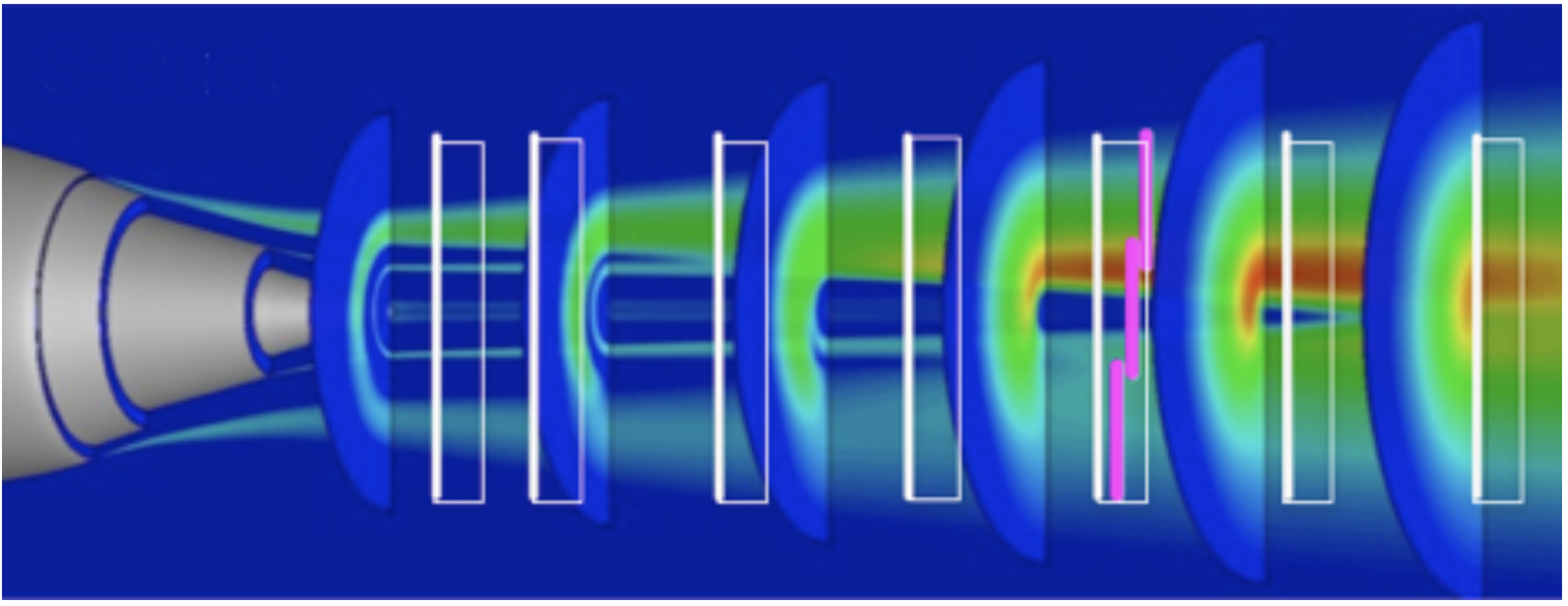}
	\caption{Visualization of axial measurement locations. White boxes show fields of view acquired on streamwise plane at jet centerline for multistream flows. There are three vertical stations at each axial locations/white box, as illustrated by the pink lines. Figure taken from \cite{Bridges2017MeasurementsOT}.}
	\label{vis:jet_flow}
\end{figure}

\begin{figure}[htb!]
	\centering
	\hspace*{-0.2cm} 
	\includegraphics[width=0.45\textwidth]{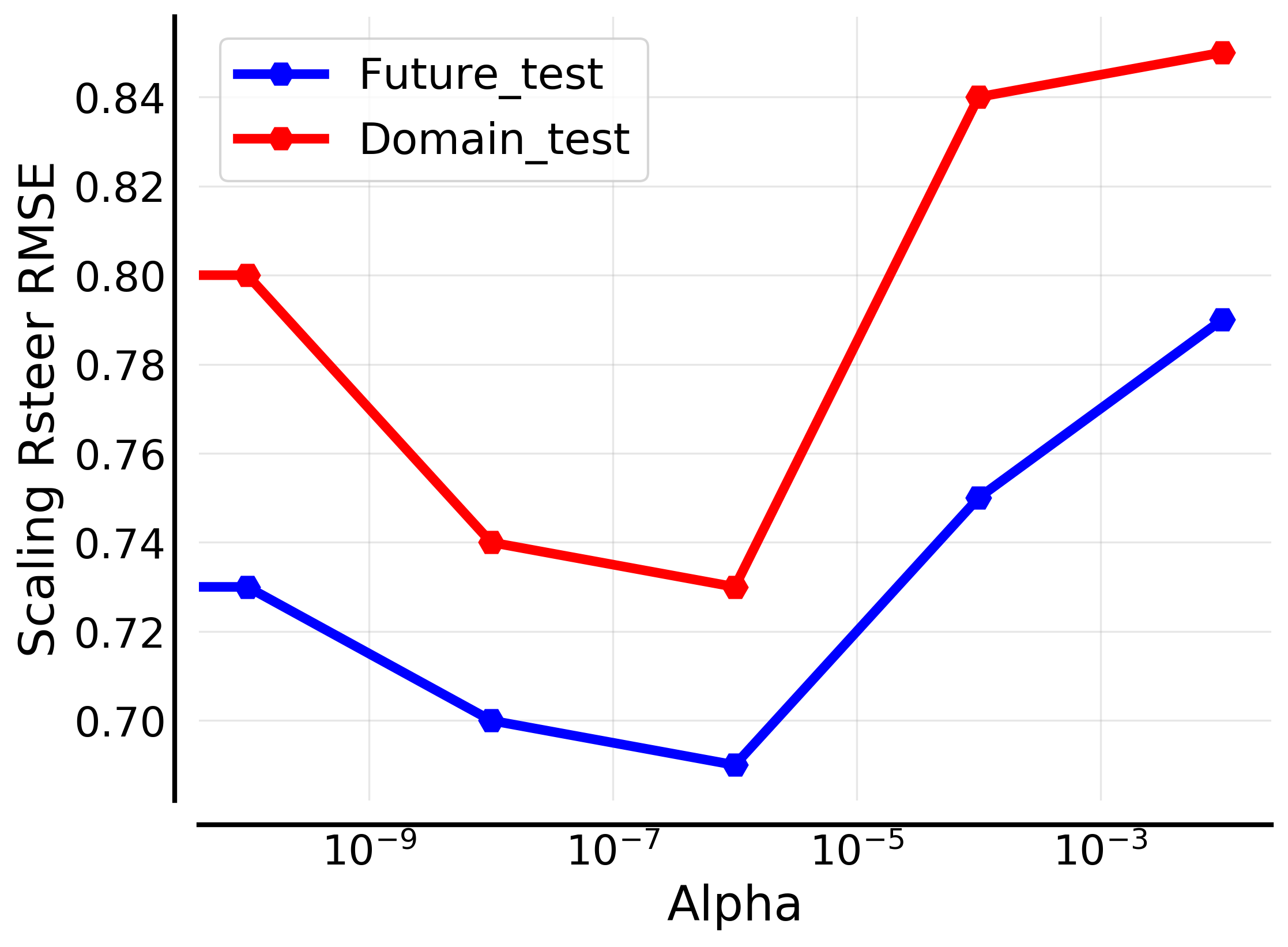}
	\caption{The prediction RMSEs on test sets of a scaling \texttt{RSteer} model trained with different regularization parameter $\alpha$}
	\label{vis:alpha}
\end{figure}


\end{document}